\documentclass[letterpaper]{article} % DO NOT CHANGE THIS
\usepackage{aaai24}  % DO NOT CHANGE THIS
\usepackage{times}  % DO NOT CHANGE THIS
\usepackage{helvet}  % DO NOT CHANGE THIS
\usepackage{courier}  % DO NOT CHANGE THIS
\usepackage[hyphens]{url}  % DO NOT CHANGE THIS
\usepackage{graphicx} % DO NOT CHANGE THIS
\usepackage{natbib}  % DO NOT CHANGE THIS AND DO NOT ADD ANY OPTIONS TO IT
\usepackage{caption} % DO NOT CHANGE THIS AND DO NOT ADD ANY OPTIONS TO IT
\usepackage{algorithm}
\usepackage{algorithmic}

\usepackage{newfloat}
\usepackage{listings}
\DeclareCaptionStyle{ruled}{labelfont=normalfont,labelsep=colon,strut=off} % DO NOT CHANGE THIS
\floatstyle{ruled}
\newfloat{listing}{tb}{lst}{}
\floatname{listing}{Listing}
\nocopyright %-- Your paper will not be published if you use this command
\title{Regret Lower Bounds \\ in Multi-agent Multi-armed Bandit}
\author {
    Mengfan Xu\textsuperscript{\rm 1},
    Diego Klabjan\textsuperscript{\rm 1}
}
\affiliations {
    \textsuperscript{\rm 1}Department of Industrial Engineering and Management Sciences,
    Northwestern University, Evanston, IL 60208, U.S.A.\\
    MengfanXu2023@u.northwestern.edu, 
    d-klabjan@northwestern.edu
}

\usepackage{amsfonts}       % blackboard math symbols
\usepackage{nicefrac}       % compact symbols for 1/2, etc.
\usepackage{microtype}      % microtypography
\usepackage{xcolor}         % colors

\usepackage{url}            % simple URL typesetting
\usepackage{booktabs}       % professional-quality tables
\usepackage{amsfonts}       % blackboard math symbols
\usepackage{nicefrac}       % compact symbols for 1/2, etc.
\usepackage{microtype}      % microtypography
\usepackage{amsthm}
\usepackage{chngcntr}
\usepackage{apptools}
\newtheorem{theorem}{Theorem}

\newtheorem*{remark}{Remark}
\newtheorem*{lemma*}{Lemma}

\newtheorem*{Proposition*}{Proposition}

\usepackage{graphicx}
\usepackage{graphics}
\usepackage{parskip}
\usepackage{amsmath}
\usepackage[ruled,vlined, algo2e]{algorithm2e}
\usepackage{subfig}

\usepackage{dsfont}
\usepackage{natbib}
\usepackage[leftcaption]{sidecap}

\usepackage[flushleft]{threeparttable}
\usepackage{array,booktabs,makecell}

\usepackage{amsmath,amssymb}
\usepackage[flushleft]{threeparttable}
\usepackage{array,booktabs,makecell}
\usepackage[font=small]{caption}

\usepackage{soul}
\newcommand\redout{\bgroup\markoverwith
{\textcolor{red}{\rule[0.5ex]{2pt}{0.8pt}}}\ULon}
\begin{document}

\maketitle

\let\thefootnote\relax\footnote{Preprint. Under review.}

\begin{abstract}
Multi-armed Bandit motivates methods with provable upper bounds on regret and also the counterpart lower bounds have been extensively studied in this context. Recently, Multi-agent Multi-armed Bandit has gained significant traction in various domains, where individual clients face bandit problems in a distributed manner and the objective is the overall system performance, typically measured by regret. While efficient algorithms with regret upper bounds have emerged, limited attention has been given to the corresponding regret lower bounds, except for a recent lower bound for adversarial settings, which, however, has a gap with let known upper bounds. To this end, we herein provide the first comprehensive study on regret lower bounds across different settings and establish their tightness. Specifically, when the graphs exhibit good connectivity properties and the rewards are stochastically distributed, we demonstrate a lower bound of order $O(\log T)$ for instance-dependent bounds and $\sqrt{T}$ for mean-gap independent bounds which are tight. Assuming adversarial rewards, we establish a lower bound $O(T^{\frac{2}{3}})$ for connected graphs, thereby bridging the gap between the lower and upper bound in the prior work. We also show a linear regret lower bound when the graph is disconnected. While previous works have explored these settings with upper bounds, we provide a thorough study on tight lower bounds. 
\end{abstract}

\section{Introduction}

Multi-armed Bandit (MAB) is a well-known online sequential decision making paradigm where a player selects arms, receives corresponding rewards at each time step, and aims to maximize their cumulative reward over a process of length $T$. Regret minimization is at the heart of MAB, where regret measures the difference between the cumulative reward obtained by always selecting the best arm and the cumulative reward achieved by a player's policy. To this end, balancing exploration (gaining information) and exploitation (maximizing current reward) is key to the player's success. Several classical algorithms have been developed for different MAB settings with proven upper bounds on the regret. Furthermore, to establish optimality of these algorithms, it is essential to prove lower bounds of the same order (in terms of the time horizon $T$) for all algorithms in specific problem instances. If such lower bounds exist, we refer to them as tight. These worst-case scenario analyses determine the fundamental complexity of bandit problems, validate whether the algorithms are optimal or not, and motivate  the development of optimal algorithms. Specifically, in the instance-dependent case, KL-divergence plays a crucial role in characterizing the hardness of distinguishing between optimal and sub-optimal arms. The seminal work by \citep{lai1985asymptotically} establishes an asymptotic regret lower bound of order $O(\log T)$ for consistent algorithms using an elegant regret decomposition approach that incorporates KL-divergence. Subsequent work relaxes the assumptions of consistency and asymptotics \citep{lattimore2020bandit} assuming 2 arms. For the mean-gap independent case, \citep{lattimore2020bandit} demonstrate a minimax regret lower bound of order $\sqrt{T}$. Furthermore, \citep{shamir2014fundamental} establishes a general regret lower bound of order $\sqrt{T}$ for MAB variants where multiple arms can be pulled at each time step. The key idea behind these results is to construct problem instances where the optimal arm is very close to the sub-optimal arms but not too close, making it challenging for the player to distinguish between them and resulting in a risk of getting less rewards and significant regret. The gap is precisely chosen and is the main technique.

Recently, the field of multi-agent Multi-armed Bandit (multi-agent MAB) has gained significant attention, driven by the application of cooperative learning processes in federated learning to various real-world scenarios, including healthcare and autonomous driving, as well as the increasing demand for large-scale distributed decision learning processes in sensor networks and robotic systems. In multi-agent MAB, multiple agents, also referred to as clients or players, face multiple MABs. The objective of the clients is to optimize the overall system performance, which is quantified using regret. Regret measures the difference between the cumulative reward obtained by pulling the optimal arm, where optimality is defined based on the average rewards across all clients, and the cumulative reward obtained by all the clients. Similar to the categorization in the traditional MAB framework, problem settings in multi-agent MAB are classified as either stochastic or adversarial, depending on the nature of reward distributions. In stochastic multi-agent MAB, the rewards for each client are independently and identically distributed over time, while in adversarial multi-agent MAB, the rewards are chosen by an adversary.

The multi-agent MAB framework presents additional challenges compared to the traditional MAB. Similar to MAB, it deals with the exploration-exploitation trade-off as a major challenge. However, in the multi-agent setting, each client faces this challenge while potentially lacking complete information about other clients. This limitation arises from the fact that optimality is defined based on average rewards across clients, requiring each client to obtain information from other clients, which, however, is constrained by the distribution of clients within the system. To tackle this issue, previous work has extensively studied settings that incorporate a central server, also referred to as a controller, as discussed in \citep{bistritz2018distributed,zhu2021federated,huang2021federated,mitra2021exploiting,reda2022near,yan2022federated}. In this setup, the central server integrates and distributes information among the clients at each time step, which has led to a regret upper bound of order $O(\log T)$ in stochastic multi-agent MAB matching the regret bounds in stochastic MAB. However, despite being mentioned in \citep{martinez2019decentralized} regarding the instance-dependent lower bound of order $\log T$, a formal lower bound statement has yet to be thoroughly examined in this centralized structure. This research gap partly motivates the present study, where we aim to address this knowledge gap and provide a comprehensive analysis of the regret lower bound within the centralized multi-agent MAB framework.

The assumption of centralization may not be realistic in real-world scenarios, where clients are often limited to pairwise transmissions constrained by underlying graph structures. In response to this, a fully decentralized framework characterized by means of graph structures has been proposed in several studies \citep{landgren2016distributed,landgren2016distributed_2,landgren2021distributed,zhu2020distributed,martinez2019decentralized,agarwal2022multi, wang2021multitask, jiang2023multi, zhu2021decentralized,zhu2021federated}. This decentralized approach removes the centralization assumption, making it more general while introducing non-trivial challenges. To this end, certain assumptions on the graphs are incorporated in these studies. Examples include complete graphs \citep{wang2021multitask}, regular graphs \citep{jiang2023multi}, and connected graphs under the doubly stochasticity assumption \citep{zhu2021decentralized,zhu2020distributed}. In all cases, the regret upper bounds that are of order $O(\log T)$, are consistent with those in the MAB setting. Furthermore, recent research has focused on time-varying graphs, such as B-connected graphs under the doubly stochasticity assumption \citep{zhu2023distributed} , as well as random graphs, including the Erdős-Rényi model and random connected graphs \citep{zhu2023distributed}. Likewise, in these cases, the regret upper bounds maintain the order $O(\log T)$. However, it is important to note that the corresponding regret lower bounds have not yet been addressed in the existing literature, which is one of the main focuses of this study.

    In a separate line of research, \citep{jia2021multi} have introduced a regret upper bound in MAB of order $\sqrt{T}$, which is independent of the sub-optimality gap $\Delta_i$ representing the difference between the mean value of the optimal arm and the mean value of the sub-optimal arms. Their setting is standard MAB. Unlike the above regret bound of order $O(\log T) = O\left(\frac{\log T}{\Delta_i}\right)$ that tends to grow rapidly when $\Delta_i$ approaches zero, this mean-gap independent regret bound remains stable even when $\Delta_i$ is very small and thereby holding universally across different problem settings. Building upon this, \citep{xu2023decentralized} analyze the decentralized multi-agent MAB framework with random graphs, and establish a regret upper bound of order $O(\sqrt{T}\log T)$, which aligns with \citep{jia2021multi} up to a logarithmic factor. However, despite these advancements in the regret upper bounds, the corresponding regret lower bounds in the mean-gap independent sense have not yet been explored. Addressing this research gap is one of the primary objectives of this paper.

In addition to the classical stochastic settings, \citep{cesa2016delay} investigate an adversarial multi-agent MAB problem and provide a regret upper bound of order $\sqrt{T}$, demonstrating its consistency with the adversarial MAB problem under the EXP3 algorithm. More recently, \citep{yi2023doubly} have focused on the heterogeneous variant, where different adversaries are different across clients. The presence of heterogeneous adversaries poses a significant challenge, resulting in a regret upper bound of order $O(T^{\frac{2}{3}})$, which is larger than the regret bound for the standard MAB problem of order $\sqrt{T}$. Furthermore, in the adversarial setting, they establish a regret lower bound of order $\sqrt{T},$ which, while informative, is smaller than their proposed regret upper bound. They achieve this by leveraging the results from the MAB setting presented in \citep{shamir2014fundamental} and constructing problem instances with mini batches of adversarial rewards. Nevertheless, it remains unexplored whether this lower bound is optimal and whether it is possible to develop even larger lower bounds or smaller upper bounds in order to claim optimality. This paper improves the lower bound in this setting and highlights its fundamental challenge by incorporating mini batches and constructing a novel graph instance. %which increases the complexity of the problem and consequently leads to a larger regret lower bound of order $T^{\frac{2}{3}}$, matching the regret upper bound of order $O(T^{\frac{2}{3}})$. This result highlights the fundamental challenge of the multi-agent MAB problem and points towards the direction of finding optimal solutions. }.

We introduce a novel contribution to the decentralized multi-agent MAB problem by investigating the regret lower bounds in various settings, accounting for different graph structures and reward assumptions. In the context of stochastic rewards and instance-dependent regret bounds, we provide the first formal analysis of the regret lower bound for the centralized setting, demonstrating its tightness. We leverage the aforementioned classical idea in MAB and incorporate it into this multi-agent MAB setting. Additionally, we conduct a comprehensive study on the regret lower bounds in decentralized settings under various graph assumptions by proposing instances that capture the problem complexities of multi-agent systems on a brand new temporal graph. We show that the regret bounds are of order $\Omega(\log T)$, aligning with the existing work's regret upper bounds and establishing their optimality and tightness.

Apart from the instance-dependent regret lower bounds of order $\Omega(\log T)$, we further extend our analysis to mean-gap independent regret lower bounds, presenting a novel contribution as well. Specifically, we establish mean-gap independent regret bounds of order $\Omega(\sqrt{T})$, which not only validate near optimality of the algorithm proposed in~\citep{xu2023decentralized} up to a $\log T$ factor but also coincide with the existing literature on MAB. This study enhances the understanding of the decentralized problem settings and provides valuable insights for future research in terms of robust methodologies in this context.

Furthermore, our research extends to adversarial settings, where we establish regret lower bounds and demonstrate their tightness across various graph assumptions, including both centralized and decentralized scenarios. Firstly, we show that the regret lower bound is of order $\Omega(\sqrt{T})$ for complete graphs, which aligns with the results for traditional MAB problems, highlighting their inherent similarities. Particularly noteworthy is our finding that the regret lower bound for decentralized multi-agent MAB with connected graphs is of order $\Omega(T^{\frac{2}{3}})$. Notably, we construct a novel graph instance in the connected graph family and adopt a more complicated random shuffling mini batches, which increases the complexity of the problem. This result effectively bridges the gap between the regret upper and lower bounds presented in~\citep{yi2023doubly} and establishes that achieving a regret upper bound of $O(\sqrt{T})$ is infeasible in this adversarial setting. Our work uncovers the inherent limitations and challenges of addressing adversarial multi-agent MAB problems even with good connectivity properties compared to traditional MAB problems. Moreover, we explore the regret lower bounds in disconnected graphs with a clique connected component and demonstrate regret lower bounds of order $\Omega(T)$. These findings provide valuable insights into the performance limitations of multi-agent MAB algorithms in graph structures with limited connectivity.

Our main contributions are as follows. We are the first
\begin{itemize}
    \item to formally establish the tight instance-dependent regret lower bounds of order $\log T$ in stochastic multi-agent MAB in both centralized and decentralized settings,
    \item to study the mean-gap independent regret lower bounds of order $\sqrt{T}$ in multi-agent MAB,
    \item to prove that for adversarial settings, the regret lower bound is of order $T^{\frac{2}{3}}$ and $T$ for connected and disconnected graphs, the first of which bridges the existing gap; a coherent analysis also extends to complete graphs, where the result is of order $\sqrt{T}$.      
\end{itemize}

The structure of the paper is as follows. First, we formally introduce the problem settings along with the notations that are utilized throughout the paper. In the subsequent section, we provide the statements on the regret lower bounds in a wide variety of settings. Finally, we summarize the paper and point out future possibilities based on the findings.

\section{Problem Formulation}

Throughout the paper, we study a decentralized system with $M \geq 3$ clients, and $T$ represents the time horizon. More specifically, the clients are labeled as nodes $1,2,\ldots,M$ on a network, where the underlying graph at each time step $1 \leq t \leq T$ is represented by an undirected graph $G_t$. It is worth emphasizing that the centralization structure is equivalent to communications on a complete graph since every pair of clients communicates through the central server.

Formally, $G_t = (V,E_t)$ is described by a unique vertex set $V = {1,2, \ldots, M}$ and an edge set $E_t$ that contains pairwise nodes and conveys the neighborhood information of $G_t$. We use $\mathcal{N}_m(t)$ to denote the neighbor set of client $m$, which represents all the neighbors of client $m$ in $G_t$. It is worth noting that the graph $G_t$ can be equivalently described by its adjacency matrix, denoted as $(X_{i,j}^t)_{1 \leq i,j \leq M}$, where the element $X_{i,j}^t$ is equal to 1 if there is an edge between clients $i$ and $j$, and 0 otherwise. For simplicity, we specify $X_{i,i} = 1$ for any client $1 \leq i \leq M$. We use $\mathcal{G}_M$ to denote the set of all connected graphs with $M$ nodes. If $G = G_t$, we call it stationary and otherwise temporal. In the Erdős-Rényi model we use superscript $c$ where $c$ is the edge probability, e.g. $\mathcal{N}_m^c(t)$ is defined based on probability $c$. In the random connected graph model we denote by $c$ the probability of an edge being in such a graph. 

Subsequently, we introduce the bandit problems associated with the clients. Consistent with the existing literature, an environment generates graphs $G_t$ and rewards $r_i^m(t)$. For each client $1 \leq m \leq M$, there are $K \geq 2$ arms to be pulled. At each time step $t$, the reward of arm $1 \leq i \leq K$ is denoted as $r_i^m(t)$, which is independently and identically distributed across time with a mean value of $\mu_i^m$. The clients draw rewards independently of one another. The interaction between the client and the environment works as follows; Client $m$ pulls an arm $a_m^t$ and obtains the corresponding reward $r_{a_m^t}^m(t)$ from the environment. Additionally, clients can communicate with their neighbors in $G_t$ as provided by the environment. This means that two clients can exchange information if and only if they are connected by an edge.

%In other words, the environment takes the tasks of generating graphs and rewards, knowing the pulled arms of the clients, and only reveals the rewards of pulled arms to clients. We call the environment fully stochastic if both the graphs and the rewards are i.i.d. across the time. Given the fully stochastic environment, 
Following the common definition of the global reward, we define the global reward of arm $i$ as $r_i(t) = \frac{1}{M}\sum_{m=1}^M r_i^m(t)$, and the corresponding expected global reward as $\mu_i = \frac{1}{M}\sum_{m=1}^M \mu_i^m$. An arm is called globally optimal if $i^* = \arg\max_{i} \mu_i$, and globally sub-optimal otherwise. The parameter $\Delta_i = \mu_{i^*} - \mu_i$ represents the sub-optimality gap of arm $i$. 

We note that $\max_{i} T \cdot \mu_i = \max_i E[\sum_{t=1}^T r_i(t)] \leq E[\max_i \sum_{t=1}^T r_i(t)]$, by the Jensen's inequality. If we establish a lower bound on the regret defined with respect to $\max_{i} T \cdot \mu_i$ (called also pseudo regret), we establish that the expected regret with respect to  $E[\max_i \sum_{t=1}^T r_i(t)]$ exhibits the same lower bound. As a result, we focus on demonstrating lower bounds on the pseudo regret throughout the paper, which is called regret for convenience.

This allows us to precisely quantify the regret associated with the action sequence (policy) $\pi = \{a_m^t\}_{1 \leq m \leq M}^{1 \leq t \leq T}$. In an ideal scenario where complete knowledge of $\{\mu_i\}_{i}$ is available, clients would prefer to pull the arm $i^*$. However, due to partially observed rewards from the bandits (dimension $i$) and limited access to information from other clients (dimension $m$), the regret of a policy $\pi$ in the bandit setting is defined as $R_T^{\pi} = T\mu_{i^*} - \frac{1}{M}\sum_{t=1}^T\sum_{m=1}^M\mu_{a_m^t}^m$. This regret metric quantifies the difference between the cumulative expected reward obtained by following the globally optimal arm and the actual reward accumulated by executing the action sequence. We consider two types of policies. Denote $\sigma_F^{t,m} = \sigma(\{\{I_j^s\}_{j \in \mathcal{N}_m(s)}\}_{s \leq t})$ where $I_j^s$ represents the information of all arms contained at client $j$ at time step $s$ and, denote $\sigma_B^{t,m} = \sigma(\{\{I_j^s(a_s^j)\}_{j \in \mathcal{N}_m(s)}\}_{s \leq t})$ where $I_j^s(a_s^j)$ represents the information of arm $a_s^j$ contained at client $j$ at time step $s$. In other words, $\sigma_F^t$ captures the history of all arms up to time $t$, whereas $\sigma_B^{t,m}$ only contains the information of client $m$'s time dependent actions up to time $t$. Henceforth, we have 
$\sigma_B^{t,m} \subset \sigma_F^{t,m}.$ 
With these notations at hand, we further define policy set $\Pi_F$ and $\Pi_B$ as  $\Pi_F = \{f_t\} \text{ where the domain of } f_t \text{ is on } \sigma_F^{t} = \{\sigma_F^{t,m}\}_m, \Pi_B = \{g_t\}  \text{ where the domain of } g_t \text{ is on } \sigma_B^{t} = \{\sigma_B^{t,m}\}_m.$ To this end we define $R_T^{B} = \min_{\pi \in \Pi_B}R_T^{\pi}$.  Likewise, assuming the observations of all arms are visible to the clients, which is referred to as the full-information setting, we denote the regret as $R_T^{F} = \min_{\pi \in \Pi_F}R_T^{\pi}$. 

The primary objective of this paper is to develop theoretical lower bounds on the regret in worst-case scenarios under different assumptions on the underlying graphs, where clients operating in decentralized settings have certain regrets regardless of the policies deployed.

\section{Lower Bound Analyses}

Before analyzing the regret lower bounds in bandit settings, we consider its relationship with the regret in the full information setting. The full information setting provides a less black-box approach for characterizing the regret of algorithms.
\begin{theorem}\label{le:t-i-0-0}
For decentralized multi-agent problems on any graph $G_t$, for all problem instances we have $R_T^{F} \leq R_T^{B}.$
\end{theorem}

\begin{proof}
Consider any policy $\pi \in \Pi_B$.   Since it only requires the information of clients' actions $\sigma_B^t$, and $\sigma_B^t \subset \sigma_F^t$, we obtain that $\pi \in \Pi_F$. Subsequently, we arrive at $\Pi_B \subset \Pi_F$ by the arbitrary choice of $\pi$, which yields that $\min_{\pi \in \Pi_F}R_T^{\pi} \leq \min_{\pi \in \Pi_B}R_T^{\pi},$
or equivalently $R_T^{F} \leq R_T^{B}$. 
\end{proof}

Subsequently, we establish the following regret lower bounds in the instance-dependent and mean-gap independent sense for the full information setting.
\begin{theorem}\label{th:t-i-0-1}
For decentralized multi-agent online problems with full information, if the graph $G$ is a complete graph, then there exists a problem instance such that 
the regret of any online distributed learning  algorithms is at least $\Omega(\sqrt{T})$ and $\Omega(\log T)$ in mean-gap independent and instance-dependent settings, respectively. 
\end{theorem}

\begin{proof}[Proof sketch]
The complete proof is presented in Appendix; we summarize the main idea as follows. We note that the complete graph case is approximately equivalent to a single-agent bandit problem with full information. For the single-agent case, there exists literature establishing the corresponding instance-dependent regret bound of order $\log T$ and mean-gap independent regret bound of order $\Omega(\sqrt{T})$, as introduced in \citep{goldenshluger2013linear} and \cite{shamir2014fundamental}, respectively.  
\end{proof}

\subsection{Instance-dependent}
Next, we demonstrate the instance-dependent lower bounds in stochastic bandits for different graph structures, building upon the previously established lower bound for the full information setting. These graph structures include time-invariant complete, connected, and regular graphs, as well as time-varying complete, connected, regular graphs, and time-varying Erdős-Rényi (E-R) model and random connected graphs, which encompass the graphs studied in prior works. The formal statement is as follows.

\begin{theorem}\label{th:t-v-1-2}
For decentralized multi-agent MAB problems with any numbers of clients and stochastic rewards, if $G_t$ are complete, or connected or regular, and either stationary or temporal, or if $G_t$ follow the E-R model or are random connected graph, then the instance-dependent expected regret $R_T^B$ of any algorithm is at least $\Omega(\log T)$. 
\end{theorem}
\begin{proof}
    The instance-dependent regret bound presents non-trivial challenges to the analysis. We start with complete graphs. We specify $K=2$ and assume $\mu_1 > \mu_2$ without loss of generality. Consider the centralized problem which has times when the clients pull the same arm (agreement) and times when the clients pull distinct arms (disagreement). We denote the number of time steps of agreement and disagreement as $T_a$ and $T_d$, respectively. We observe that $T_a + T_d = T$. For $T_d$, there exist clients pulling the worse arm, which implies that for any policy $\pi \in \Pi_B$
\begin{align}\label{eq:E[R_T]}
        R_T^{\pi} & = \frac{1}{M}\sum_{m}\sum_{t \in T_d}(\mu_1 - \mu_{a_t^m}) + \frac{1}{M}\sum_{m}\sum_{t \in T_a}(\mu_1 - \mu_{a_t^m}) \notag \\
        & = \sum_{t \in T_d}\Delta_2 + \frac{1}{M}\sum_{m}\sum_{t \in T_a}(\mu_1 - \mu_{a_t^m}) \notag \\
        & = T_d\Delta_2 + \frac{1}{M}\sum_{m}\sum_{t \in T_a}(\mu_1 - \mu_{a_t^m}).
    \end{align}
Note that when $T_d = \Omega(\log T)$, we immediately derive that $E[R_T^B] \geq \Omega(\log T)$, which concludes the proof. 

From now on, we assume $T_d = o(\log T)$, which implies that $T_a = T  - o(\log T)$ and $\frac{T_a}{T} \to 1$ as $T$ goes to $\infty$. We denote the value $t_0 = \log T$ and divide the time horizon into $\overset{t_0}{\underset{j=0}{\cup}} [2^j,2^{j+1}-1]$. It is clear that 1) the number of intervals is $\log T$ and 2) the length of the $j^{th}$ interval is $2^{j-1}$. Let $t_d = \max\{t|t \in T_d\} + 1$. Since $T_d = o(\log T)$, we have $|[t_d,T]| \geq 2^{\frac{1}{2}\log T}$ for all large enough $T$. %Since $T_d = o(\log T)$, it holds that when $j$ is large enough, for the time interval, there will be no time steps belonging to $T_d$, which is denoted as $[t_d, T]$ that has a length at least $2^{\frac{1}{2}\log T}$. 

Meanwhile, we observe that for $T_a$, it is equivalent to a single-agent multi-objective bandit problem~\citep{xu2023pareto} since the global reward of a single arm $i$ is given as a reward vector $(r_i^{m,t})_{m=1}^{M}$ and is revealed to all the clients at each time step.

Note that $ \frac{1}{M}\sum_{m}\sum_{t \in T_a}(\mu_1 - \mu_{a_t^m}^{m}) = \frac{1}{M}\sum_{m}\sum_{t \in T_a}(\mu_1 - \mu_{a_t}^{m})  = \sum_{t \in T_a}(\mu_1 - \mu_{a_t})$
where the first equality is by the definition of $T_a$ and the second equality uses the definition of $\mu_1$ and $\mu_{a_t}$. 
We denote $T_a^d  = T_a \cap [t_d, T] = [t_d, T]$.%, which is a continuous time horizon.  

At the same time, the Pareto pseudo regret reads $R_{T_a^d,M} = Dist(\sum_{t \in T_a^d}(\mu_{a_t}^m)_{m}, O)$
where $Dist(\cdot)$ is the distance measure between a reward vector and the Pareto optimal set $O$ as introduced in~\citep{xu2023pareto}, and satisfies that $R_{T_a^d,M} \geq \Omega(\log T_a^d)$
for any policy $\{a_t\}$ based on Theorem 6 in~\citep{xu2023pareto}. 

By specifying the rewards homogeneous, i.e. $\mu_{a_t}^{1} = \mu_{a_t}^{2} = \ldots = \mu_{a_t}^{M}$ and following a similar analysis as on Theorem 6 in~\citep{xu2023pareto}, we obtain $R_{T_a^d,M} =  Dist(\sum_{t \in T_a^d}(\mu_{a_t}^m)_{m}, O) = \sum_{t \in T_a^d}(\mu_1 - \mu_{a_t})$
which yields
\begin{align}\label{eq:E_R_T_2}
   & \sum_{t \in T_a}(\mu_1 - \mu_{a_t}) \notag \geq \sum_{t \in T_a^d}(\mu_1 - \mu_{a_t}) \notag \\
   & \geq \Omega(\log T_a^d) = \Omega(\log (2^{\frac{1}{2}\log T})) = \Omega(\log T).
\end{align}

To put everything together, we have that for any policy $\pi \in \Pi_B$ $R_T^{\pi} \geq T_d\Delta_2 + \frac{1}{M}\sum_{m}\sum_{t \in T_a}(\mu_1 - \mu_{a_t^m})) \geq \Omega(\log T)$
where the second inequality holds by~\eqref{eq:E_R_T_2}. 

Subsequently, we obtain $\min_{\pi \in \Pi_B}R_T^{\pi}  \geq T_d\Delta_2 + \frac{1}{M}\sum_{m}\sum_{t \in T_a}(\mu_1 - \mu_{a_t^m})) \geq \Omega(\log T)$, which concludes the analysis of complete graphs. 

The remaining cases follow from the monotonicity of the regret in the graph complexity as follows. We first consider the full-information setting. For any $0 < c \leq 1$, we denote $\sigma_c^t = \sigma(\{\{I_j^s\}_{j \in \mathcal{N}_m^c(s)}\}_{s \leq t})$. We observe that $\sigma_1^t = \sigma(\{I_1^s, \ldots , I_M^s\}_{s \leq t})$. We have $\sigma_c^t \subset \sigma_1^t. $ We define policy set $\Pi_c$ as $\{f_t\} \text{ where the domain of } f_t \text{ is on } \sigma_c^{t-1}.$ 

For any policy $\pi \in \Pi_c$, i.e. $\pi = \{h_t\}_{t=1}^T$, we have that it only leverages the neighborhood information $\sigma_c^{t-1}$ to determine a decision rule at each time step. Since $\sigma_c^{t-1} \subset \sigma_1^{t-1}$, $\sigma_1^{t-1}$ also has the neighborhood information that $h_t$ requires. This leads to 
$\pi \in \Pi_1$, and subsequently yields $\Pi_c \subset \Pi_1$. We hence obtain that in the full-information setting $\min_{\pi \in \Pi_1}R_T^{\pi} \leq \min_{\pi \in \Pi_c}R_T^{\pi}. $

By the above discussion on $c$ and the statement for complete graphs, or equivalently, with respect to $\Pi_1$, we obtain $\Omega(\log T) \leq \min_{\pi \in \Pi_1}R_T^{\pi},$
in the instance-dependent sense and subsequently $\Omega(\log T) \leq \min_{\pi \in \Pi_c}R_T^{\pi}.$

By Theorem~\ref{le:t-i-0-0}, we have $R_T^{B} \geq \Omega(\log T).$ This completes the E-R case. All remaining cases follow the same logic. 

\end{proof}

\begin{remark} While \citep{martinez2019decentralized} discuss the instance-dependent regret lower bound of order $\Omega(\log T)$ in the centralized setting, we provide the first formal statement for various graphs. The result coincides with the lower bound in the single-agent MAB setting. Furthermore, the result is consistent with the established upper bounds in the multi-agent MAB settings, thereby demonstrating its tightness.   
\end{remark}

\iffalse
\begin{theorem}\label{th:t-v-1-2}
For decentralized multi-agent MAB problems with any numbers of clients and stochastic rewards, if $G_t$ is complete, connected or regular, then the instance-dependent regret of any algorithms is at least $O(\log T)$. 
\end{theorem}

\begin{theorem}\label{th:t-v-1-1}
For decentralized multi-agent MAB problems with any numbers of clients and stochastic rewards, for any constant $0 \leq c \leq 1$, there exists a graph sequence $\{G_t\}_{1 \leq t \leq T}$ belonging to the E-R model or random connected graphs, and a problem instance such that the instance-dependent regret of any algorithms is at least $O(\log T)$. 
\end{theorem}
\fi

Additionally, we also consider scenarios with disconnected graphs, which can result in linear regret due to the presence of isolated clients when the rewards are heterogeneous. The first result applies to consistent algorithms, following the classical assumption made in some existing literature. The consistency assumption states that the regret of the considered algorithms is of order $o(T^a)$ for any constant $0 < a \leq 1$. The second result applies to any algorithms, with the constraint of limiting the number of arms to 2. These results are summarized in the following statements.

\begin{theorem}\label{th:t-i-2}
For decentralized multi-agent MAB problems, if graph $G$ is disconnected with a clique connected component, then there exists a problem instance such that 
the regret of any online distributed algorithms that are individually consistent at local clients is at least $\Omega(T)$. 
\end{theorem}

\begin{proof}[Proof sketch]
    The proof is deferred to Appendix; the main logic is as follows when the clique is an isolated vertex. We construct a problem instance as follows. For clients $1,\ldots, M-1$, their reward distributions are the same, reading as $(\Delta, 0,\ldots, 0) \in R^K$, while for client $M$, the reward distribution reads as $(0, 2\Delta, 0,\ldots, \ldots, 0) \in R^K$ for any $\Delta >0$. We assume node $M$ is isolated. Using any consistent algorithms at client $M$ leads to $E[n_{M,2}(T)] = \Omega(T)$ and subsequently results in a linear regret. Here $n_{M,2}$ is the number of pulls of arm $2$ at client $M$.  
\end{proof}

As mentioned earlier, we remove the consistency assumption by assuming the number of clients is 2, which essentially deals with the trade-off between the problem setting and the considered algorithms. 

\begin{theorem}\label{th:t-i-2-r}
For decentralized multi-agent MAB problems, if graph $G$ is  disconnected with a clique connected component, then there exists a problem instance with $K=2$ such that 
the regret of any online distributed algorithms is at least $\Omega(T)$. 
\end{theorem}

\begin{proof}[Proof sketch]
The proof is given in Appendix; the proof logic is as follows when the clique component is an isolated vertex. We again let client $M$ be an isolated node. For two arms labeled as arm $1$ and $2$, we construct the instance at clients as follows. Let random variable $x$ follow a uniform distribution in $\{0,1\}$ and be fixed once determined, and for any time step $t$, the reward $r_k^j(t)$ is generated as $r_k^1(t)=\begin{cases}
			x & \text{arm 1}\\
            \frac{1}{2} & \text{arm 2}
		 \end{cases}$
and for $j > 1$ we have $r_k^j(t)=\begin{cases}
			\frac{1}{2} & \text{arm 1}\\
            \frac{1}{2} & \text{arm 2.}
		 \end{cases}$ The randomness of $x$ changes the optimality of arms, and makes client $M$ even harder to identify the global optimal arm and impossible to achieve sublinear regret even though inconsistent algorithms are deployed. 
\end{proof}

\begin{remark} To the best of our knowledge, this is the first result on the regret lower bound for settings with disconnected graphs. This linear regret essentially highlights the inherent complexity of multi-agent MAB problems compared to their single-agent counterparts.
\end{remark}

\subsection{Mean-gap independent}

Apart from the instance-dependent regret lower bounds, we also investigate the mean-gap independent regret lower bound that is applicable to both stochastic and adversarial settings. The regret order in this case is $\sqrt{T}$, which differs from the $\log T$ bound. The following theorem summarizes these results, considering all the previously mentioned graph structures.

\begin{theorem}\label{th:t-v-1-2}
For decentralized multi-agent MAB problems with any numbers of clients and stochastic rewards, if $G_t$ are complete, connected or regular, and stationary or temporal, or the E-R model or random connected graphs, then the mean-gap independent regret of any algorithm is at least $\Omega(\sqrt{T})$. 
\end{theorem}
\begin{proof}[Proof sketch]
    The formal proof is in Appendix; the main logic is as follows. The proof is similar to that of Theorem~\ref{th:t-v-1-2}, except that we consider mean-gap independent bounds using Theorem 4 in \citep{shamir2014fundamental}. We first analyze settings with complete graphs and establish $R_T^{B} \geq \sqrt{\frac{KT}{1 + M}} = \Omega(\sqrt{T})$. Likewise, the monotonicty in graphs of the regret bounds allow us to determine the same result for other graphs, which concludes the proof. 
\end{proof}
\iffalse
\begin{theorem}\label{th:t-v-1-2}
For decentralized multi-agent MAB problems with any numbers of clients and stochastic rewards, if $G$ is complete, connected or regular, then the mean-gap independent regret of any algorithms is at least $O(\sqrt{T})$. 
\end{theorem}

\begin{theorem}\label{th:t-v-1-2}
For decentralized multi-agent MAB problems with any numbers of clients and stochastic rewards, if $G_t$ is complete, connected or regular, then the mean-gap independent regret of any algorithms is at least $O(\sqrt{T})$. 
\end{theorem}

\begin{theorem}\label{th:t-v-1-1}
For decentralized multi-agent MAB problems with any numbers of clients and stochastic rewards, for any constant $0 \leq c \leq 1$, there exists a graph sequence $\{G_t\}_{1 \leq t \leq T}$ belonging to the E-R model or random connected graphs, and a problem instance such that the mean-gap independent of any algorithms is at least $O(\sqrt{T})$. 
\end{theorem}
\fi

\begin{remark} Similarly, this result aligns with the lower bound established in the single-agent MAB setting. Furthermore, this lower bound of order $\sqrt{T}$ corresponds to the mean-gap upper bounds presented in ~\citep{xu2023decentralized} and \citep{jia2021multi} for multi-agent and single-agent MAB problems, respectively. This consistency further shows the tightness of the lower bound we have derived.
\end{remark}

\subsection{Adversarial}
Since the mean-gap independent regret bounds hold for the stochastic problem setting, they also hold for the adversarial problem setting. This is due to the fact that the set of stochastic settings is essentially a subset of the set of adversarial settings. Therefore, our result remains consistent with the result in~\citep{yi2023doubly}.

\begin{theorem}\label{th:t-i-0}
For decentralized multi-agent MAB problems, if the graph $G_t$ is a complete graph, then there exists a problem instance such that 
the regret of any online distributed learning  algorithms is at least $\Omega(\sqrt{T})$. 
\end{theorem}

%\begin{proof}
%    This holds directly by Theorem~\ref{th:t-v-1-2} when $G_t$ is complete.
%\end{proof}

Furthermore, we construct special connected graphs, in adversarial settings and demonstrate that they lead to a regret lower bound of order $\Omega(T^{\frac{2}{3}})$. This bound is larger than the commonly observed $O(T^{\frac{1}{2}})$ in single-agent adversarial settings and decentralized multi-agent adversarial settings with complete graphs. We summarize these results in the following two theorems, one for a large number of clients and the other one for a small number of clients.

\begin{theorem}\label{th:t-i-1}
For decentralized multi-agent MAB problems, if the number of clients $M \geq \Omega(T^{\frac{1}{3}})$ and the graph $G_t$ is a connected graph with two expanders of size $\frac{M}{4}$ having distance $d \geq \frac{\eta M}{8}$ given constant $4> \eta > 0$, then there exists a problem instance such that 
the regret of any online distributed learning  algorithm is at least $\Omega(T^{\frac{2}{3}})$. 
\end{theorem}
%\hl{Probably could consider the adaptive adversary and translate the setting to this multi-agent network.}
\begin{proof}[Proof sketch]
   The proof is deferred to Appendix; the idea is summarized as follows. We consider clients are distributed on a special connected graph, e.g. a path graph and focus on two subsets of node, denoted as $I_0$ and $I_1$, respectively, that satisfy $|I_0| = |I_1| = \frac{M}{4}$, and the shortest path $d_p$ from $I_0$ to $I_1$ meets the condition $d_p \geq \frac{\eta M}{8}.$ Then the choice of $M$ gives $d_p \geq \Omega(T^{\frac{1}{3}})$ and we import the result in \citep{yi2023doubly} and obtain 
    $R_T^B \geq \Omega(\sqrt{d_p \cdot T}) = \Omega(T^{\frac{2}{3}})$ for full-information settings.% Using Theorem~\ref{le:t-i-0-0} leads to 
    %$\min_{\pi}R_T^{\pi} \geq \Omega(T^{\frac{2}{3}})$. 
\end{proof}
\begin{remark}
    Note that the existence of such graphs is guaranteed by the property of expanders of size $\frac{M}{4}$. An expander of size $\frac{M}{4}$ has a diameter of order $\log M$ (Proposition 3.1.5 in \citep{kowalski2019introduction}). Indeed, for $\eta = 4$, a path is such an expander. %This implies that the diameter of $G_t$ would not exceed $\log M + d + \log M \leq 2\log M + \frac{M}\eta{8} \leq M $. Therefore, the graph $G_t$ is well defined.  
\end{remark}

For small values of $M$, achieving the same regret lower bound requires additional effort since the setting allows for more communication between clients. In this case, we present the following result that establishes the same lower bound on regret by importing techniques from information theory.

\begin{theorem}\label{th:t-i-1.2}
For decentralized multi-agent MAB problems, if the number of clients $M = T^{\frac{2}{15}}$ and the graph $G_t$ is a connected graph with two expanders of size $\frac{M}{4}$ having distance $d \geq \frac{\eta M}{8}$ given constant $4> \eta > 8 \cdot 8^{-\frac{2}{15}}$, then there exists a problem instance such that 
the regret of any online distributed learning  algorithms is at least $\Omega(T^{\frac{2}{3}})$. 
\end{theorem}

\begin{proof}
Let $M \, mod \, 4 = 0$ and $T > 8$. Denote expanders of size $\frac{M}{4}$ as two disjoint subsets of nodes $I_0 = \{1,2,\ldots,\frac{M}{4}\}$ and $I_1 = \{\frac{3}{4}M,\frac{3}{4}M+1,\ldots,M\}$. Note that $|I_0| = |I_1| = \frac{M}{4}$. By the definition of $G_t$, the shortest path distance between $I_0$ and $I_1$ is $d \geq \frac{\eta M}{8}$. We set $\epsilon = \sqrt{\frac{4}{\eta}}\frac{M^2}{2}T^{-\frac{1}{3}}$. It follows $8\epsilon^2d \leq 1$.%$d \leq \frac{M}{32\epsilon^2} = \frac{MT^{\frac{2}{3}}}{8}$.   

Let $B_1$ be Bernoulli with probability $\frac{1}{2} + \epsilon$ and $B_2$ Bernoulli with probability $\frac{1}{2}$.  Consider the bandit problem as follows. Let $X$ be a random variable following a uniform distribution on $\{0,1,\ldots,\frac{M}{4}\}$. For client $X \geq 1$, arm 1 follows $B_1$ and arm 2 follows $B_2$. For $i \in I_0 \backslash \{X\}$, let the arms follow $B_2$. All clients not in $I_0$ have all rewards $0$.  %he mean value of reward follows the Bernoulli distribution satisfying $\mu_k^X(t)=\begin{cases}
			%\frac{1}{2}+\epsilon & \text{arm $k$ = 1}\\
            %\frac{1}{2} & \text{arm $k$ = 2.}
		 %\end{cases}$ If $X=0$, then $\mu_k^j(t)=\begin{cases}
			%\frac{1}{2} & \text{arm $k=$ 1}\\
            %\frac{1}{2} & \text{arm $k=$ 2}
%\end{cases}$ for every $j \in I_0$. 
%For all other clients in $I_0| \{X\}$ we set $\mu_k^j(t)=\begin{cases}
%			\frac{1}{2} & \text{arm $k=$ 1}\\
%            \frac{1}{2} & \text{arm $k=$ 2}
%\end{cases}$ and $\mu_k^j(t)= 0$ for $k = 1,2$ and $j \not\in I_0$. %We assume that the reward of arm $k$ at client $j$ at time step $t$ is a Bernoulli random variable with mean value $\mu_k^i(t)$ over the sample space $\{0,1\}$. 
   
Additionally, we re-sample random variable $X$ every $d$ steps, i.e. we re-specify the client $X$ if $X \geq 1$. If $X=0$, all clients have reward based on $B_2$. We denote the number of such re-sampling steps as $D$, $D = \lfloor\frac{T}{d}\rfloor$, which leads to a sequence $\{X_1, X_2, \ldots, X_D\}$. The following holds for $i \in I_0$. Subsequently, let us define distribution %let us define the conditional distribution on $X_j$ with respect to the joint distribution of arm $1$ and $2$, denoted as 
$Q^i_j(arm) = P( arm|X_j = i)$  and $Q^{-1}_j(arm) = P(arm | X_j = 0)$. Note that $Q^{-1}_j$ represents that all clients in $I_0$ share the same reward distribution. Let $Q^i_{j,t}(arm)=P(arm |\sigma_t,X_j = i)$ and $Q^{-1}_{j,t}(arm) = P(arm |\sigma_t, X_j = 0)$. %, which represents the joint distribution of arm 1 and 2 at the current time step. 
It is easy to verify that 
\begin{align*}
    & D_{KL}(Q^{-1}_{j,t}, Q^{i}_{j,t}) = \frac{1}{2}\log \frac{\frac{1}{2}}{\frac{1}{2}-\epsilon} +   \frac{1}{2}\log \frac{\frac{1}{2}}{\frac{1}{2}+\epsilon}\\ & = \frac{1}{2}\log (1 + \frac{4\epsilon^2}{1 - 4\epsilon^2}) \leq \frac{1}{2} \cdot \frac{4\epsilon^2}{1 - 4\epsilon^2}  \leq 4\epsilon^2,
\end{align*}
\iffalse
\begin{align*}
    & D_{KL}(Q^{i}_{j,t}(\cdot,\cdot), Q^{-1}_{j,t}(\cdot,\cdot)) \\
    & = (\frac{1}{2})^2\log \frac{(0.5)^2}{(0.5)(0.5-\epsilon)} + (\frac{1}{2})^2\log \frac{(0.5)^2}{(0.5)(0.5-\epsilon)} + \\
    & \qquad (\frac{1}{2})^2\log \frac{(0.5)^2}{(0.5)(0.5+\epsilon)} + (\frac{1}{2})^2\log \frac{(0.5)^2}{(0.5)(0.5-\epsilon)} \\ & = \frac{1}{2}\log (1 + \frac{4\epsilon^2}{1 - 4\epsilon^2}) \\
    & \leq \frac{1}{2} \cdot \frac{4\epsilon^2}{1 - 4\epsilon^2}  \leq 4\epsilon^2
\end{align*}
\fi
where the first inequality uses the fact that $\log (1+x) \leq x$ and the second inequality holds by the choice of $\epsilon = \frac{M^2}{2}T^{-\frac{1}{3}} \leq \frac{1}{4}$ since $T > 8$.

Therefore, by the chain rule for relative entropy, we obtain $D_{KL}(Q^{-1}_j, Q^{i}_j) = \sum_{t = jd}^{(j+1)d}D_{KL}(Q^{-1}_{j,t}, Q^{i}_{j,t}) \leq \sum_{t = jd}^{(j+1)d}4\epsilon^2 \leq 4\epsilon^2d. $

By the Pinsker's inequality we have that $  D_{TV}(Q^{-1}_j, Q^{i}_j) \leq \sqrt{\frac{D_{KL}(Q^{-1}_j, Q^{i}_j)}{2}} \leq \epsilon\sqrt{2d}.\hfill \refstepcounter{equation}(\theequation)\label{eq:d_t_v}$

The expected reward of arm 1 is $\frac{1}{8} + \frac{1}{M}\frac{|I_0|}{|I_0|+1}\epsilon$ from
\begin{align*}
    & \mu_1= \frac{1}{M}\sum_{m=1}^M\mu_1^{m} = \frac{1}{M}\sum_{m\in I_0}\mu_1^{m} +  \frac{1}{M}\sum_{m \not\in I_0}\mu_1^{m} \\
    & = \frac{1}{M}\sum_{m\in I_0}\Big[E[\mu_1^{m}|X_1 \in I_0]P(X_1 \in I_0) + \\
    & \qquad \sum_{m\in I_0}E[\mu_1^{m}|X_1 \not\in I_0]P(X_1 \not\in I_0)\Big] +  \frac{1}{M}\sum_{m \not\in I_0}0 \\
    & = \frac{1}{M}(\frac{|I_0|}{|I_0|+1} (\frac{1}{2}+\epsilon + \frac{1}{2}(|I_0|-1)) + \\
    & \qquad \frac{1}{|I_0|+1}(\frac{1}{2} + \frac{1}{2}(|I_0|-1))  \\
    & = \frac{1}{8} + \frac{1}{M}\frac{|I_0|}{|I_0|+1}\epsilon
\end{align*}
and of arm 2 is $\frac{1}{8}$ from 
\begin{align*}
    & \mu_2 = \frac{1}{M}\sum_{m=1}^M\mu_2^m \\
    & = \frac{1}{M}\sum_{m \in I_0}\mu_2^m + \frac{1}{M}\sum_{m \not\in I_0}\mu_2^m \\
    & = \frac{1}{M}\sum_{m \in I_0}\frac{1}{2} + \frac{1}{M}\sum_{m \not\in I_0}0 = \frac{1}{8}.
\end{align*}

As a result $\Delta_1 = \frac{\epsilon}{M} \frac{|I_0|}{|I_0|+1} \geq  \frac{\epsilon}{2M}$ since $|I_0| \geq 1$. Let us denote by $n_{m,1}(T,j)$ the number of pulls of arm $1$ by client $m$ during the $j^{th}$ epoch which is the optimal arm. Therefore, we obtain
\begin{align}\label{eq:E_R_T}
    & E[R_T^B] = E[E[R_T^B|X_1, \ldots, X_D]] \notag \\
    & = E[E[\frac{1}{M}\sum_{m=1}^{M}(\frac{\epsilon}{2M}(T - n_{m,1}(T)))|X_1, \ldots, X_D]] \notag \\
    & = E[E[\frac{1}{M}\sum_{m=1}^{M}(\frac{\epsilon}{2M}(\sum_{j=1}^Dd - \sum_{j=1}^Dn_{m,1}(T,j)))|X_1, \ldots, X_D]] \notag \\
    & = E[\frac{1}{M}\sum_{m=1}^{M}\sum_{j=1}^DE[(\frac{\epsilon}{2M}(d - n_{m,1}(T,j)))|X_1, \ldots, X_D]] \notag \\
    & = \frac{1}{M}\sum_{m=1}^{M}\sum_{j=1}^DE[E[(\frac{\epsilon}{2M}(d - n_{m,1}(T,j)))|X_j]] \notag \\  
    & = \frac{1}{M}\sum_{m=1}^{M}\sum_{j=1}^D\sum_{i 
 \in I_0 \cup \{0\}}\frac{E[(\frac{\epsilon}{2M}(d - n_{m,1}(T,j)))|X_j = i]]}{|I_0|+1}\notag \\
    %& \geq \frac{1}{M}\sum_m(\frac{1}{|I_0|}\sum_j\textcolor{red}{\sum_i}E[\frac{\epsilon}{M} \cdot (d - n_{m,1}(T,j)) | X_j = \textcolor{red}{i}]) \notag \\
    & \geq \frac{1}{2M^2}(\frac{1}{|I_0| + 1}\sum_{j=1}^D\sum_{i \in I_0 \cup \{0\}}E[\epsilon \cdot (d - n_{1,1}(T,j)) | X_j = i]) \notag \\
    & = \frac{1}{2M^2}(\epsilon \cdot T - \frac{\epsilon}{|I_0| + 1}\sum_{j=1}^D\sum_{i 
 \in I_0 \cup \{0\}}E_{Q^{i}_j}[(n_{1,1}(T,j))])
\end{align}
where the the first and fifth equality use the  law of total expectation, the third equality is by the fact that $T = \sum_{j=1}^Dd$ and $\sum_{j=1}^Dn_{m,1}(T,j) = n_{m,1}(T)$, and the sixth equality uses the distribution of $X_j$ defined by $P(X_j = i) = \frac{1}{|I_0|+1}$ for $i \in I_0 \cup \{0\}$.

 Note that $E_{Q^{i}_j}[(n_{1,1}(T,j))] - E_{Q^{-1}_j}[(n_{1,1}(T,j))] = \sum_{t = jd}^{(j+1)d}(Q^{i}_j(a_t^{1} = 1) - Q^{-1}_j(a_t^{1} = 1)) \leq d \cdot D_{TV}(Q^{-1}_j,Q^{i}_j)$
where the last inequality is by the definition of the total variation $D_{TV}$. 

This immediately gives us that
\begin{align*}
    & \sum_{i \in I_0 \cup \{0\}}\sum_{j=1}^DE_{Q^{i}_j}[(n_{1,1}(T,j))] \\
    & \leq \sum_{i \in I_0 \cup \{0\}}\sum_{j=1}^D\sum_{t = jd}^{(j+1)d}(Q^{-1}_j(a_t^{1} = 1) + d \cdot D_{TV}(Q^{i}_j,Q^{-1}_j)) \\
    & \leq T + d\sum_{i \in I_0 \cup \{0\}}\sum_{j=1}^DD_{TV}(Q^{i}_j,Q^{-1}_j)) \\
    & \leq T + d\sum_{i \in I_0 \cup \{0\}}\sum_{j=1}^D(\epsilon\sqrt{2d}) \\
    & = T + dD\epsilon\sqrt{2d}(|I_0|+1)  = T + T \cdot \frac{|I_0|+1}{4}
\end{align*}
where the second inequality uses $\sum_iQ^{-1}_j(a_t^{1} = 1) = 1$ and $dD = T$, and the third inequality uses~\eqref{eq:d_t_v}, and the last equality holds by the choices of $d$ and $\epsilon$ that satisfy $\epsilon\sqrt{2d}(|I_0|+1) \leq \frac{|I_0|+1}{4}$. Here we also use the lower bound on $\eta$. 

Consequently, we arrive at
\begin{align}
    E[R_T^B]
    & \geq \frac{1}{2M^2}(\epsilon \cdot T -  \frac{\epsilon}{|I_0|+1}(T + T \cdot \frac{|I_0|+1}{4})) \notag \\
    & \geq \frac{1}{2M^2} \frac{1}{4}\epsilon \cdot T = \Omega(T^{\frac{2}{3}})
\end{align}
where the last inequality uses $|I_0| = \frac{M}{4} \geq 2$ and the equality holds by the choice of $\epsilon$ and $M$.
\end{proof}

\begin{remark} It is worth noting that this lower bound is consistent with the regret upper bound in~\citep{yi2023doubly}, bridging the gap between the regret upper bound $O(T^{\frac{2}{3}})$ and the lower bound $\Omega(\sqrt{T})$ in \citep{yi2023doubly}. Surprisingly, it also coincides with the regret lower bound for online learning with feedback graphs in~\citep{alon2015online}, where the feedback received by the client is limited to a graph structure. This connection highlights the relationship between the decentralized multi-agent MAB system and MAB with side information on graphs. Lastly, we observe that this bound is larger than $\sqrt{T}$ in the single-agent MAB, manifesting the fundamental difference between multi-agent and single-agent MAB in the presence of connected graphs, in addition to the settings with disconnected graphs.
\end{remark}

\section{Conclusion}
In this paper, we conduct a comprehensive study on the regret lower bounds in a decentralized multi-agent MAB framework across various settings, which provides an understanding of the fundamental challenges posed by different problem settings and insights into the development of optimal algorithms. Specifically, we establish instance-dependent and mean-gap independent lower bounds for stochastic settings, which are of order $\log T$ and $\sqrt{T}$, respectively, for all existing graphs. These results are consistent with the existing upper and lower bounds, showing their tightness and consistency, respectively. Additionally, we introduce a novel problem instance in adversarial settings that leads to a regret lower bound of order $\Omega(T^{\frac{2}{3}})$. This finding bridges the gap between the existing lower and upper bounds and highlights the distinction between the multi-agent and single-agent counterparts. Furthermore, we uncover worst-case scenarios in multi-agent MAB settings by demonstrating a linear regret when the graphs are disconnected, which adds to the difference between multi-agent and single-agent MAB. As a next step, we suggest exploring novel algorithms with smaller coefficients that are close to the lower bounds established herein. As a concluding remark, how to show high probability lower bounds remain an important yet unexplored area of research.

\bibliography{aaai24}

\appendix
\section{Proof of Results in Section 4}
\subsection{Proof of Theorem 2}
\begin{proof}
   On a complete graph, each client can observe the rewards of all arms at $M$ clients, where the number of observations is thereby upper bounded by $KM$. Henceforth, we consider Theorem 4 in \citep{shamir2014fundamental} to obtain 
   \begin{align*}
         R_T^{F} \geq \sqrt{\frac{KT}{1 + KM}} = \Omega(\sqrt{T}).
   \end{align*}
   This completes the first part of the statement.

   For the instance-dependent regret lower bounds, we assume that the number of arms is 2 and the rewards of arms satisfies the assumptions in \citep{goldenshluger2013linear}. Then based on the result established by specifying a contextual linear bandit with $\alpha = 1$ as in \citep{goldenshluger2013linear}, which reads as Theorem 2, we obtain 
   \begin{align*}
       R_T^{F} \geq \Omega(\log{T}).
   \end{align*}
   We add that the lower bound result for the bandit setting holds for the full-information setting by noting the analysis essentially uses the observations that are given by the full information setting. 

   This concludes the instance-dependent lower bound in the full information setting and thereby completes the proof. 

\end{proof}

\subsection{Proof of Theorem 4}

\begin{proof}[Proof]
    Consider a disconnected graph $G$ with a clique connected component $C_G$ including clients $c_1, \ldots, c_{Q}$ without loss of generality. %In other words, the sub-graph induced by clique $C_G$ is a complete graph. 
    Since $G$ is disconnected, for any other node $m \not\in V(C_G)$, there is no path between $m$ and any node in $C_G$.  

    Let $\Delta > 0$. For client $m \not\in C_G$, the reward distributions read as $(\frac{M-1}{M - Q}\Delta, 0,\ldots, 0)$, which indicates that the optimal arm is arm $1$. For client $m \in C_G$, however, the reward distribution reads as $(0, \frac{2}{Q}\Delta, 0,\ldots, \ldots, 0)$, implying that arm $2$ is the optimal arm. It is straight-forward that the global mean reward value of arm 1 is $\frac{(M-1)}{M}\Delta$ that is larger than that of arm $2$ which is $\frac{2\Delta}{M}$. The subsequent sub-optimality gap is $\Delta_2 = \frac{M-3}{M}\Delta$. Any no-regret (consistent as proposed in~\citep{lattimore2020bandit}) algorithms $\pi$ at client $j \in C_G$, where the regret with respect to the available information is defined on the rewards of client $j \in C_G$, leads to $E[n_{j,2}(T)] = O(T)$. However, in this situation, the global regret satisfies
    \begin{align*}
        E[R_T^{\pi}] & = \frac{1}{M}\sum_{m}\sum_{t =1}^T(E[\mu_1 - \mu_{a_t^m}]) \\
        & \geq \frac{1}{M}\sum_{t =1}^T(E[\mu_1 - \mu_{a_t^j}]) \\
        & \geq \frac{1}{M}E[n_{j,2}(T)] \cdot \Delta_1 \\
        & =  \frac{1}{M} \cdot \frac{M-3}{M}\Delta \cdot \Omega(T) = \Omega(T)
    \end{align*}
    where the first inequality is by only considering client $j$ and the second inequality uses the fact that arm $2$ is not a global optimal arm.  

    This completes the proof of the linear regret in the case when clients perform local consistent learning on disconnected graphs.    
\end{proof}

\subsection{Proof of Theorem 5}

\begin{proof}[Proof]
Again, we consider a disconnected graph $G$ with a clique $C_G$ including clients $c_1, \ldots, c_{Q}$ without loss of generality.

We assume there are two arms labeled as arm $1$ and $2$ and consider the instance at clients as follows by referencing~\citep{alon2015online}. Let random variable $X$ follow a uniform distribution in $\{0,1\}$ and be fixed once determined, and for any time step $t$, the reward $r_k^j(t)$ is generated as for any $j \not\in C_G$, $r_k^j(t)=\begin{cases}
			X & \text{arm 1}\\
            \frac{1}{2} & \text{arm 2}
		 \end{cases}$
and for any $j \in C_G$, we have $r_k^j(t)=\begin{cases}
			\frac{1}{2} & \text{arm 1}\\
            \frac{1}{2} & \text{arm 2}
		 \end{cases}$
where the random variable $X$ is independent of everything at client $j \in C_G$ as client $j \in C_G$ only has the information of their own arms. We have $\Delta_2 = \frac{1}{2(M - Q)},$
no matter what value $X$ takes since it only changes the choice of optimal arms. Specifically, when $X = 1$, the global optimal arm is arm $1$ and the suboptimality gap is $\Delta_2  = \mu_1 - \mu_2 = (1 - \frac{1}{2})/(M - Q)$. When $X = 0$, the global optimal arm is arm $2$ and the suboptimality gap is $\Delta_2  = \mu_2 - \mu_0 = (\frac{1}{2} - 0)/(M - Q)$, the other way around. 

Subsequently, we consider the regret at client $j \in C_G$ to obtain 
\begin{align*}
   & E[R_T^{\pi}] =  \frac{1}{M}\sum_{m}\sum_{t =1}^T(E[\mu_{*} - \mu_{a_t^m}])  \\
   & \geq \frac{1}{M}\sum_{t =1}^T(E[\mu_{*} - \mu_{a_t^M}]) \\
   & = \frac{1}{M}(\frac{1}{2}E[\Delta n_{j,1}(T)| X = 0] + \frac{1}{2}E[ \Delta (T-n_{j,1}(T)) | X = 1]) \\
   & = \frac{1}{M}(\frac{1}{2}E[\Delta n_{j,1}(T)] + \frac{1}{2}E[ \Delta (T-n_{j,1}(T))]) \\
   & = \frac{\Delta}{4M(M - Q)}T = \Omega(T)
\end{align*}
where the first inequality uses the non-negativity of value $\mu_{*} - \mu_{a_t^m}$ and the third equality leverages the independence between $X$ and client $j$. 
\end{proof}

\subsection{Proof of Theorem 6}
\begin{proof}
We show the mean-gap free regret lower bound starting with complete graphs. Note that a complete graph is equivalent to a centralized problem with $M$ agents. This implies that each client can observe the reward of multiple arms by communicating with $M-1$ neighbors, where the number of observations is thereby upper bounded by $M$. Henceforth, we consider Theorem 4 in \citep{shamir2014fundamental} and obtain 
    \begin{align*}
        R_T^{B} \geq \sqrt{\frac{KT}{1 + M}} = \Omega(\sqrt{T}).
    \end{align*}
    This completes the proof of the complete graphs. 

        Regarding the monotonicity of the regret in the graph complexity, the proof follows the proof of Theorem 3.

\end{proof}

\subsection{Proof of Theorem \ref{th:t-i-1}}

\begin{proof}
    Note that the graph structure determines the communication efficiency of the clients. To consider the lower bound, we leverage sparse graphs in the connected graph family to perform the worst-case scenario analysis. 

Specifically, we consider the designed graph consisting of clients $1, \ldots, M$ in this order. It takes exactly $O(M)$ time steps for client 1 to obtain the information of client $M$, which results in a deterministic delay. 

If $I_0 = \{1, \ldots, \frac{M}{4}\}$ and $I_1 = \{\frac{3M}{4}, \ldots, M\}$, then the shortest path $d_p$ from $I_0$ to $I_1$ meets the condition 
\begin{align*}
    d_p \geq \Omega(\frac{M+1}{3}).
\end{align*}

By the choice of $M$ such that $M > \Omega(T^{\frac{1}{3}})$, we obtain 
\begin{align}\label{eq:d_p}
    d_p \geq \Omega(T^{\frac{1}{3}}).
\end{align}
%for a given $\Tilde{\Delta} = \frac{\sqrt{2}}{3}\sqrt{\frac{\lambda_1}{\lambda_{M-1}}}$ being the mixing time of the weight matrix $W$ associated with $G$.

We star with a full-information setting. Following  a similar argument and constructing the same instance as in Lemma A.4 in~\citep{yi2023doubly}, we arrive that in the full-information setting
\begin{align*}
    R_T \geq \Omega(\sqrt{d_p \cdot T}).
\end{align*}

Subsequently, we obtain that
\begin{align*}
    R_T & \geq \Omega(\sqrt{d_p \cdot T}) \\
    & = \Omega(\sqrt{T} \cdot \sqrt{d_p}) \\
    & \geq \Omega(\sqrt{T} \cdot T^{\frac{1}{6}}) = \Omega(T^{\frac{2}{3}})
\end{align*}
where the last inequality is by (\ref{eq:d_p}). Equivalently, we write it as 
\begin{align}\label{eq:r_t_f_2_3}
    R_T^{F} \geq \Omega(T^{\frac{2}{3}}).
\end{align}

Meanwhile, by Lemma~\ref{le:t-i-0-0}, we have that the regret lower bound in the bandit setting is larger than the regret in the full information setting
and thus by (\ref{eq:r_t_f_2_3}) we obtain
\begin{align*}
   R_T^{B} \geq \Omega(T^{\frac{2}{3}}).
\end{align*}

This completes the proof of Theorem~\ref{th:t-i-1}.

\end{proof}

\end{document}